\documentclass[pdflatex,sn-mathphys]{sn-jnl}% Math and 
\usepackage{bm}
\usepackage{multirow}
\newcommand{\ar}[1]{\textcolor{black}{#1}}

\jyear{2022}%

\newtheorem{theorem}{Theorem}
\newtheorem{lemma}{Lemma}%  meant for continuous numbers
\newtheorem{proposition}[theorem]{Proposition}% 
\newtheorem{example}{Example}
\newtheorem{remark}{Remark}%
\newtheorem{definition}{Definition}%

\begin{document}

\title[T-Norms Driven Loss Functions for Machine Learning]{T-Norms Driven Loss Functions for \\ Machine Learning}


\author*[1]{\fnm{Francesco} \sur{Giannini}}\email{francesco.giannini@unisi.it}
\equalcont{These authors contributed equally to this work.}

\author[2]{\fnm{Michelangelo} \sur{Diligenti}}\email{michelangelo.diligenti@unisi.it}
\equalcont{These authors contributed equally to this work.}

\author[2]{\fnm{Marco} \sur{Maggini}}\email{marco.maggini@unisi.it}

\author[2,3]{\fnm{Marco} \sur{Gori}}\email{marco.gori@unisi.it}

\author[4]{\fnm{Giuseppe} \sur{Marra}}\email{giuseppe.marra@kuleuven.be}
\equalcont{These authors contributed equally to this work.}

\affil[1]{\orgdiv{Consorzio Interuniversitario Nazionale per l'Informatica}, \orgname{CINI}, \orgaddress{\city{Roma}, \state{Italy}}}

\affil[2]{\orgdiv{Department of Information Engineering and Science}, \orgname{University of Siena}, \orgaddress{\city{Siena}, \state{Italy}}}

\affil[3]{\orgdiv{Maasai, Inria, I3S, CNRS}, \orgname{Universit$\hat{\mbox{e}}$ C$\hat{\mbox{o}}$te d'Azur}, \orgaddress{\city{Nice}, \state{France}}}

\affil[4]{\orgdiv{Department of Computer Science}, \orgname{KU Leuven}, \orgaddress{\city{Leuven}, \state{Belgium}}}

\abstract{\ar{Injecting prior knowledge into the learning process of a neural architecture is one of the main challenges currently faced by the artificial intelligence community, which also motivated the emergence of neural-symbolic models. One of the main advantages of these approaches is their capacity to learn} competitive solutions with a significant reduction of the amount of supervised data.
\ar{In this regard, a commonly adopted solution consists of} representing the prior knowledge via first-order logic formulas, then relaxing the formulas into a set of differentiable constraints by using a t-norm fuzzy logic.
\ar{This paper shows that this relaxation, together with the choice of the penalty terms enforcing the constraint satisfaction, can be unambiguously determined by the selection of a t-norm generator, providing numerical simplification properties and a tighter integration between the logic knowledge and the learning objective.}
When restricted to supervised learning, the presented theoretical framework provides a straight derivation of the popular cross-entropy loss, which has been shown to provide faster convergence and to reduce the vanishing gradient problem in very deep structures.
However, the proposed learning formulation extends the advantages of the cross-entropy loss to the general knowledge that can be represented by neural-symbolic methods. In addition, the presented methodology allows the development of novel classes of loss functions, which are shown in the experimental results to lead to faster convergence rates than the approaches previously proposed in the literature.}

%%================================%%
%% Sample for structured abstract %%
%%================================%%

\keywords{Learning from constraints, T-norm generators, Loss functions, Integration of logic and learning, Neural-symbolic integration.}

\maketitle

\section{Introduction}
\label{sec:introduction}

Deep Neural Networks~\cite{lecun2015deep} have been a break-through for several classification problems involving sequential or high-dimensional data. 
However, deep neural architectures strongly rely on a large amount of labeled data to develop powerful feature representations.
Unfortunately, it is difficult and labor intensive to annotate such large collections of data. \ar{In this regard, prior knowledge expressed by} First-Order Logic (FOL) \ar{rules}
represents a natural solution to make learning efficient when the training data is scarce and some \ar{domain expert} knowledge is available.
The integration of logic inference with learning could also overcome another limitation of deep architectures, \ar{namely} that they mainly act as black-boxes from a human perspective, making their usage difficult in safety critical applications, \ar{like in health or car industry applications \cite{selbst2018meaningful}.}
For these reasons, Neural-Symbolic \ar{(NeSy)} approaches~\cite{de2021statistical,garcez2019neural} integrating logic and learning \ar{have become one of the fundamental research lines for the} machine learning and artificial intelligence communities.
One of the most common approaches to exploit logic knowledge to train a deep neural learner relies on mapping the FOL knowledge into differentiable constraints using t-norms. Then, the constraints can be enforced using gradient-based optimization techniques, like done in~\cite{diligenti2017semantic,badreddine2022logic}. Most work in this area approached the problem of translating logic rules into a differentiable form by defining a collection of heuristics that often lack semantic consistency and have no clear motivation from a theoretical point of view. For instance, there is no agreement on the relation between the selected t-norm and the aggregation \ar{function corresponding to the logic quantifiers, nor even on the chosen loss to enforce the constraints.}

This paper first traces back the properties of t-norm fuzzy logic operators down to the selection of a generator function. Then, we show \ar{that the loss function of a learning problem accounting for both supervised data and logic constraints can also be determined by} the \ar{single choice} of the \textit{t-norm generator}. \ar{The generator determines the fuzzy relaxation of connectives and quantifiers occurring in the logic rules. As a result, a simplified} and semantically consistent \ar{optimization problem can be formulated. In this framework,} the classical fitting of supervised training data can be enforced by atomic logic constraints. Since the careful choice of loss functions has been crucial to the success of deep learning, this paper also investigates the relation between supervised training losses and generator choices.
\ar{As a special case, we get} a novel justification for the popular cross-entropy loss~\cite{goodfellow2016deep}, that has been shown to provide faster convergence and to reduce the vanishing gradient problem in very deep structures.

\noindent\ar{\textbf{Contributions. } This paper introduces a theoretical framework centered around the notion of t-norm generator, unifying the choice of the logic semantics and of the loss function in neural-symbolic learners. In particular, we extend the preliminary formalization sketched in \cite{giannini2019relation}, together with a more comprehensive experimental validation.
This unification results in a simplified learning objective that is shown to be numerically more stable, while retaining the flexibility to customize the learning process on the considered applications.}

The paper is organized as follows: Section~\ref{sec:related_works} presents some prior work on the integration of learning and logic inference, Section~\ref{sec:FL} presents the basic concepts about t-norms, generators and aggregator functions and Section~\ref{sec:nsl} introduces a general neural-symbolic framework used to extend supervised learning with logic rules.
Section \ref{sec:Lgen} presents the main results of the paper, showing the link between t-norm generators and loss functions and how these can be exploited in neural-symbolic approaches. Section~\ref{sec:exp} presents the experimental results \ar{and a discussion on the presented methodology is provided in Section~\ref{sec:discussion}. Finally, Section~\ref{sec:conc} draws some conclusions.}

%%%%%%%%%%%%%%%%%%%%%%%%%%%%%%%%%%%%%%%%%%% 
\section{Related Works}
\label{sec:related_works}
Neural-symbolic approaches~\cite{garcez2022neural,hitzler2022neuro} aim at combining symbolic reasoning with (deep) neural networks, \ar{e.g. by exploiting additional logic knowledge when available.} This knowledge can be \ar{either injected into} the learner internal structure (e.g. by \ar{constraining} the network architecture) or \ar{enforced on the learner} outputs (e.g. by \ar{adding new} loss terms). 
\ar{In this context,} First-Order Logic is commonly chosen as the declarative framework \ar{to represent the} knowledge because of its flexibility and expressive power.
NeSy methodologies are rooted in previous work on Statistical Relational Learning (SRL)~\cite{raedt2016statistical,de2021statistical}, which developed frameworks for performing logic inference in presence of uncertainty. For instance, Markov Logic Networks (MLN)~\cite{richardson2006markov} and Probabilistic Soft Logic (PSL)~\cite{bach2017hinge} integrate FOL and probabilistic graphical models \ar{by using the logic rules as potential functions defining a probability distribution. MLNs have received a lot of attention by the SRL community~\cite{niu2011tuffy,chekol2016markov,qu2019gmnn} and have been widely used in different tasks like information extraction, entity resolution and text mining~\cite{khot2015exploring,gayathri2017probabilistic}.
% One , since they allow to perform both inference and parameter learning in the same framework.
More recently, MLNs have also been extended to work with neural potential functions in \cite{marra2021neural}, showing impressive results e.g. in generating molecular data. PSL can be considered a fuzzy extension of MLNs, as it exploits a fuzzy relaxation of the logic potentials by using \L ukasiewicz Logic. The framework proposed in this paper builds upon t-norm fuzzy logics, however it is not limited to any specific t-norm. Hence it could be also adopted to define alternative logic potential functions for PSL.}

A common solution to integrate logic reasoning and deep learning relies on using deep neural networks to approximate the truth values (i.e. fuzzy semantics) or the probabilities (i.e. probabilistic semantics) of certain target predicates, \ar{and then apply logic or probabilistic inference on the network outputs \cite{diligenti2021constraint}. In the former case, the logic rules can be relaxed according to a differentiable fuzzy logic and then the overall architecture  can be optimized end-to-end.}  This approach is followed with minor variants by Semantic-Based Regularization (SBR)~\cite{diligenti2017semantic}, Lyrics~\cite{marra2019lyrics} and Logic Tensor Networks (LTN)~\cite{badreddine2022logic}, especially for classification problems.
\ar{On the other hand, some examples of NeSy} approaches based on probabilistic logic \ar{are given by} Semantic Loss~\cite{xu2018semantic}, Differentiable Reasoning~\cite{van2019semi}, Deep Logic Models~\cite{marra2019integrating}, Relational Neural Machines~\cite{marra2020relational} and DeepProbLog~\cite{manhaeve2018deepproblog}.
Similarly, Lifted Relational Neural Networks~\cite{sourek2018lifted} and Neural Theorem Provers~\cite{rocktaschel2017end,minervini2020learning} realize a soft forward or backward chaining via an end-to-end gradient based scheme.
\ar{This paper investigates the bound between the selected logic semantics to represent the knowledge and the loss function in the learning task. This is a common problem for all NeSy approaches, that encode the logic knowledge into differentiable constraints used by a deep learner.}

\ar{\paragraph{Learning with Fuzzy Logic Constraints}}
\ar{In general, if some FOL knowledge is available for a learning problem, this is expressed in Boolean form. To define a differentiable learning objective is then fundamental to establish a mapping to relax the logic formulas into differentiable functional constraints by means of an appropriate fuzzy logic.}
For instance, Serafini et al. \cite{serafini2017learning} introduces a learning framework where the formulas are converted according to the t-norm and t-conorm of \L ukasiewicz logic. Giannini et al. \cite{giannini2018convex} also proposes to convert the formulas according to \L ukasiewicz logic, however they exploit the weak conjunction in place of the t-norm, \ar{thus guaranteeing} convex functional constraints. A more empirical approach has been considered in SBR, where all the fundamental t-norms have been evaluated on different learning \ar{settings to select the best t-norm on the single tasks ~\cite{diligenti2017semantic}.}
More recent studies on the learning properties of different fuzzy logic operators have also been  proposed by Van Krieken et al. \cite{van2020analyzing,van2022analyzing}. \ar{By combining different logic semantics for the connectives, the authors achieved the most significant performance improvement, but the dependence between the connectives is no longer obeying any specific logic theory. }

\ar{The relaxation of logic quantifiers has also been the subject of a wide range of studies.
On the performance side, different quantifier conversions have been} taken into account and validated. For instance, in Diligenti et al.~\cite{diligenti2017semantic} the arithmetic mean and the maximum operator have been used to convert the universal and existential quantifiers, respectively. Different possibilities have been considered for the universal quantifier in Donadello et al.~\cite{donadello2017logic}, while the existential quantifier depends on this choice via the application of the strong negation using the DeMorgan law. \ar{However,} the arithmetic mean operator has been shown to achieve better performances in the conversion of the universal quantifier~\cite{donadello2017logic}, with
the existential quantifier implemented by Skolemization. 
\ar{In spite of improving the performances,} the universal and existential quantifiers \ar{should} be thought of as a generalized AND and OR, respectively. Therefore, converting these quantifiers using a mean operator has no direct justification inside a logic theory, \ar{and spoil the original semantics}.

% In spite of the extensive prior work in this area, it has not emerged a 
% unified principle \ar{on how to select a suitable loss function that could be logically interpreted according to the adopted} fuzzy logic \ar{semantics}.
%In particular, all the aforementioned NeSy approaches rely on fixed loss functions  measuring the distance of the formulas from the \ar{$1$ (True)} value. Even if in some cases this may be justified from a logical point of view, since e.g. the strong logical negation is defined as $\neg x=1-x$, it is not clear whether this choice is principled from a learning standpoint, since all deep learning approaches use very different loss functions, e.g. to enforce the fitting of the supervised data. This paper traces back the properties of t-norm fuzzy logic operators down to the single choice of a generator function, also highlighting the connection with a meaningful choice of the loss function.
\ar{There have been a few attempts in the literature to address the problem of choosing semantically driven loss functions to enforce the satisfaction of the logic constraints. 
However, these works are generally not fully semantically coherent or too specific. 
A unified principle to select a suitable loss function that can be logically interpreted according to the adopted fuzzy logic semantics is still missing.
For instance, both SBR \cite{diligenti2017semantic} and LTN \cite{serafini2017learning} rely on minimizing the strong negation of each logic constraint, whereas Lyrics \cite{marra2019lyrics} also allows the usage of the negative logarithm. A different perspective is considered in Semantic Loss \cite{xu2018semantic}, where the authors propose a new loss function that is very close to the negative logarithm one and that is able to achieve (near) state-of-the-art performances on semi-supervised learning tasks, by combining neural networks and logic constraints.
In this paper, we show that these loss functions (and infinitely many more) are special cases of t-norm generators that can be uniquely determined by the choice of a fuzzy logic relaxation.}

% %%%%%%%%%%%%%%%%%%%%%%%%%%%%%%%%%%%%%%%%%%%%%%%%%%
\section{Background on T-Norm Fuzzy Logic}
\label{sec:FL}
Many-valued logics have been introduced in order to extend the admissible set of truth values from \emph{true} ($1$) \ar{and} \emph{false} ($0$) to a scale of truth-degree having \emph{absolutely true} and \emph{absolutely false} as boundary cases.
A fuzzy logic is a many-valued logic, whose set of truth values coincides with the real unit interval $[0,1]$. This section introduces the basic notions of fuzzy logic together with some illustrative examples.

T-norms \cite{klement2013triangular} are a special kind of binary operations on the real unit interval $[0, 1]$, representing an extension of the Boolean conjunction.

\begin{definition}\label{def:tnorm}
	$T:[0,1]^2\to[0,1]$ is a t-norm if and only if for every $x,y,z\in[0,1]$:
	\[
	\begin{array}{l}
	T(x,y)=T(y,x),~~T(x,T(y,z))=T(T(x,y),z),\\
	T(x,1)=x, ~~ T(x,0)=0, ~~\mbox{if }
	x\leq y~\mbox{ then }~ T(x,z)\leq T(y,z) \ .
	\end{array}
	\]
	$T$ is a continuous t-norm if \ar{it is a continuous} 
 function in $[0,1]$.
\end{definition}

\ar{A fuzzy logic can be uniquely defined according to the choice of a certain t-norm $T$~\cite{hajek2013metamathematics}. A wide variety of operations corresponding to different fuzzy logic connectives are defined starting from $T$ and the strong negation ``$\neg$", and their notation is introduced in Definition~\ref{def:conn}. Table~\ref{tab:fuzzySem}} reports the algebraic semantics of these connectives for G\"{o}del, \L ukasiewicz and Product logics, which are referred as the fundamental fuzzy logics, because all the continuous t-norms can be obtained from them by ordinal sums~\cite{jenei2002note}.
\begin{definition}\label{def:conn}
	\[
	\begin{array}{lr}
	(\mbox{t-norm}) &  x\otimes y = T(x,y)\\
	(\mbox{residuum}) &  x\Rightarrow y = \max\{z:\,x\otimes z\leq y\} \\
	(\mbox{bi-residuum}) & x\Leftrightarrow y = (x\Rightarrow y)\otimes(y\Rightarrow x)\\
	(\mbox{weak conjunction}) & x\wedge y =x\otimes (x\Rightarrow y) \\ 
	(\mbox{weak disjunction}) & x\vee y \!=\! ((x\!\Rightarrow\! y)\Rightarrow y)\otimes((y\Rightarrow x)\Rightarrow x) \\ 
	(\mbox{residual negation}) & \sim x = x\Rightarrow0 \\ 
	(\mbox{strong negation}) & \neg x = 1-x \\ 
	(\mbox{t-conorm}) & x\oplus y = \neg( \neg x\otimes\neg y) \\ 
	(\mbox{material implication}) & x\rightarrow y = \neg x\oplus y 
	\end{array}
	\]
\end{definition}

\begin{table}[h]
	\small
 \caption[Truth-functions for the fundamental fuzzy logics]{The truth functions for the \ar{t-norm,} residuum, bi-residuum, weak conjunction, weak disjunction, residual negation, strong negation, t-conorm and material implication of the fundamental fuzzy logics.}
	\label{tab:fuzzySem}
	\centering
	\begin{tabular}{|c|c|c|c|}
		\hline
		& G\"{o}del & \L ukasiewicz & Product \\
		\hline
		\hline
		$x\otimes y$ & $\min\{x,y\}$ & $\max\{0,x+y-1\}$ & $x\cdot y$ \\
		\hline
		$x\Rightarrow y$ & $x\leq y?1:y$ & $\min\{1,1-x+y\}$ & $\min\{1,\frac{y}{x}\}$\\
    % $x\leq y?1:\frac{y}{x}$ \\
		\hline
		$x\Leftrightarrow y$ & $x\leq y?x:y$ & $1-\mid x-y\mid$ &
		$x=y?1:\min\{\frac{x}{y},\frac{y}{x}\}$
		\\
		\hline
		$x\wedge y$ & $\min\{x,y\}$ & $\min\{x,y\}$ & $\min\{x,y\}$ \\
		\hline
		$x\vee y$ & $\max\{x,y\}$ & $\max\{x,y\}$ & $\max\{x,y\}$ \\
		\hline
		$\sim x$ & $x=0?1:0$ & $1-x$ & $x=0?1:0$ \\
		\hline
		$\neg x$ & $1-x$ & $1-x$ & $1-x$ \\
		\hline
		$x \oplus y$ & $\max\{x,y\}$ & $\min\{1,x+y\}$ & $x+y-x\cdot y$ \\
		\hline
		$x \rightarrow y$ & $\max\{1-x,y\}$ & $\min\{1,1-x+y\}$ & $1-x+x\cdot y$ \\
		\hline
	\end{tabular}
\end{table}

\subsection{Archimedean T-Norms}
\label{sec:gent-norm}

\ar{Continuous} Archimedean t-norms~\cite{klement2013triangular} are special t-norms that can be constructed by means of unary monotone functions, called \emph{generators}.
\begin{definition}
	A t-norm $T$ is Archimedean if for every $x\in(0,1)$ it \ar{holds} $T(x,x)<x$. $T$ \ar{is said to be strict} if for all $x\in(0,1)$ we have $0<T(x,x)<x$, \ar{otherwise it is said to be nilpotent}.
\end{definition}
\noindent \ar{For example,} \L ukasiewicz  ($T_L$) and Product ($T_P$) t-norms are nilpotent and strict respectively, while G\"{o}del ($T_G$) t-norm is idempotent (i.e. $\forall x:\ T_G(x,x)=x$) and \ar{hence not even} Archimedean. \ar{In addition, all the nilpotent and strict t-norms can be related to the \L ukasiewicz and Product t-norms as follows.}

\begin{theorem}[\cite{klement2013triangular}]
	Any nilpotent t-norm is isomorphic to $T_L$ and any strict t-norm is isomorphic to $T_P$.
\end{theorem}

\noindent\ar{The next theorem} \ar{shows how to} construct t-norms by \emph{additive\footnote{Since here we only deal with additive generators, we will drop the term ``additive" for simplicity.} generators} \cite{klement2013triangular}.
\begin{theorem}
	Let $g:[0,1]\to[0,+\infty]$ be a strictly decreasing function with $g(1)=0$  and $g(x)+g(y)\in Range(g)\cup\{g(0^+),+\infty\}$ for all $x, y$ in $[0, 1]$, and $g^{(-1)}$ its pseudo-inverse. Then the function $T:[0,1]^2\to[0,1]$ defined as
	\begin{equation}\label{eq:t-normgen}
	T(x,y)=g^{-1}\left(\min\{g(0^+),g(x)+g(y)\}\right) 
	\end{equation}
	is a t-norm and $g$ is said an additive generator for $T$. \ar{Moreover,} $T$ is strict if $g(0^+)=+\infty$, otherwise $T$ is nilpotent.
\end{theorem}

\begin{example}
	If we take $g(x)=1-x$, we get the \L ukasiewicz t-norm $T_L$.
	\[
	T(x,y)=1-\min\{1,1-x+1-y\}=\max\{0,x+y-1\} 
	\]
	\end{example}
	
	\begin{example}
	If we take $g(x)=-\log(x)$, we get the Product t-norm  $T_P$.
	\[
	T(x,y)=\exp\big( -(\min\{+\infty,-\log(x)-\log(y)\}) \big) = x\cdot y 
	\]
\end{example}
\ar{According to Equation (\ref{eq:t-normgen}), the other fuzzy logic connectives deriving from the t-norm can be expressed with respect to the generator.} For instance: 
\begin{eqnarray}\label{eq:genres}
x\Rightarrow y &=& g^{-1}\left(\max\{0,g(y)-g(x)\}\right) \nonumber\\ 
x\Leftrightarrow y &=& g^{-1}\left(\mid g(x)-g(y) \mid \right) \\
x\oplus y &=& 1-g^{-1}\left(\min\{g(0^+),g(1-x)+g(1-y)\}\right) \nonumber
\end{eqnarray}

\subsection{Parameterized Classes of T-Norms}
\ar{T-norm generators can also depend on a parameter, by consequently defining a parameterized class of t-norms.} For instance, given a generator $g$ of a t-norm $T$ and $\lambda>0$, then $T^{\lambda}$ denotes a class of increasing t-norms that correspond to the generator function $g^{\lambda}(x)=(g(x))^{\lambda}$. In addition, let $T_D$  and $T_G$ denote the Drastic ($T_D(x,y)=(x=y=1)?1:0$) and G\"{o}del t-norms respectively, we get:
\[
\lim_{\lambda\rightarrow0^+}T^{\lambda}=T_D\qquad\mbox{and}\qquad \lim_{\lambda\rightarrow\infty}T^{\lambda}=T_M  
\]

\ar{Over the years,} several parameterized families of t-norms have been introduced and studied  in the literature \cite{mizumoto1989pictorial,klement2013triangular}. In the following, we recall some prominent examples that we will exploit in the experimental evaluation.

\begin{definition}[The Schweizer-Sklar family]
	\label{ex:ss}
	For $\lambda\in(-\infty,+\infty)$, consider:
	\[
	g_{\lambda}^{SS}(x)=
	\left\{
	\begin{array}{lr}
	-\log(x) & \mbox{if }\lambda=0 \\
	\frac{1-x^{\lambda}}{\lambda} & \mbox{otherwise}
	\end{array}
	\right.
	\]
	The t-norms corresponding to this generator are called  Schweizer-Sklar t-norms, and they are defined according to:
	\[
	T_{\lambda}^{SS}(x,y)=
	\left\{
	\begin{array}{lr}
	T_G(x,y) & \mbox{if }\lambda=-\infty \\
	(x^{\lambda}+y^{\lambda}-1)^{\frac{1}{\lambda}} & \mbox{if }-\infty<\lambda<0 \\
	T_P(x,y) & \mbox{if }\lambda=0 \\
	\max\{0,x^{\lambda}+y^{\lambda}-1\}^{\frac{1}{\lambda}} & \mbox{if }0<\lambda<+\infty \\
	T_D(x,y) & \mbox{if }\lambda=+\infty
	\end{array}
	\right.
	\]
	The Schweizer-Sklar t-norm $T_{\lambda}^{SS}$ is Archimedean if and only if $\lambda>-\infty$, continuous if and only if $\lambda<+\infty$, strict if and only if  $-\infty<\lambda\leq 0$ and nilpotent if and only if $\, 0<\lambda<+\infty$. This t-norm family is strictly decreasing for $\lambda\geq0$ and continuous with respect to $\lambda\in[-\infty,+\infty]$, in addition $T_1^{SS}=T_L$.
\end{definition}

\begin{definition}[Frank t-norms]
	\label{ex:frank}
	For $\lambda\in[0,+\infty]$, consider: 
	\[
	g_{\lambda}^F(x)=
	\left\{
	\begin{array}{lr}
	-\log(x) & \mbox{if }\lambda=1 \\
	1-x & \mbox{if }\lambda=+\infty \\
	\log(\frac{\lambda-1}{\lambda^x-1}) & \mbox{otherwise}
	\end{array}
	\right.
	\] 
	The t-norms corresponding to this generator are called  Frank t-norms and they are strict if $\lambda<+\infty$. The overall class of Frank t-norms is decreasing and continuous.
	\[
	T_{\lambda}^F(x,y)=
	\left\{
	\begin{array}{lr}
	T_G & \mbox{if }\lambda=0\\
	T_P & \mbox{if }\lambda=1\\
	T_L & \mbox{if }\lambda=+\infty\\
	\log_{\lambda}\left(1+\frac{(\lambda^x-1)(\lambda^y-1)}{\lambda-1}\right) & \mbox{otherwise}
	\end{array}
	\right.
	\]
\end{definition}

% %%%%%%%%%%%%%%%%%%%%%%%%%%%%%%%%%%%%
\section{Background on the Integration of Learning and Logic Reasoning}
\label{sec:nsl}

According to the \emph{learning from logical constraints} paradigm~\cite{diligenti2021constraint}, \ar{the available prior knowledge is represented by a set of logic rules. which are relaxed into continuous and differentiable constraints over the task functions (implementing FOL predicates). Positive and negative supervised samples can also be seen as atomic constraints, and the learning process corresponds to} finding the task functions that best satisfy the constraints.

\ar{\begin{example}
Let us assume that the prior knowledge for an image classification task is expressed by the following sentences ``lions live in savanna or in zoos'' and ``there are no walls in the savanna'' (see Figure \ref{fig:lion}). This domain knowledge can be represented in FOL as ``$\forall x\ Lion(x)\rightarrow LiveIn(x,savanna)\vee LiveIn(x,zoo)$" and ``$\forall x\  Wall(x)\rightarrow \neg LiveIn(x,savanna)$", being $Lion,Wall$ two unary predicates, $LiveIn$ a binary predicate and $savanna,zoo$ two constants. If a neural classifier is able to correctly detect the presence of a lion and a wall in Figure \ref{fig:lion}, it is also able to establish that the lion is living in a zoo by exploiting the symbolic knowledge.
\end{example}
}
% Symbolic logic provides a natural way to express factual and abstract knowledge about a problem by means of logical formulas, \ar{and, as already mentioned in Section \ref{sec:related_works}, a large class of methodologies convert this knowledge into fuzzy valued constraints.}
\begin{figure}[t] %QUI CHECK!!
    \centering
    \includegraphics[width=5cm]{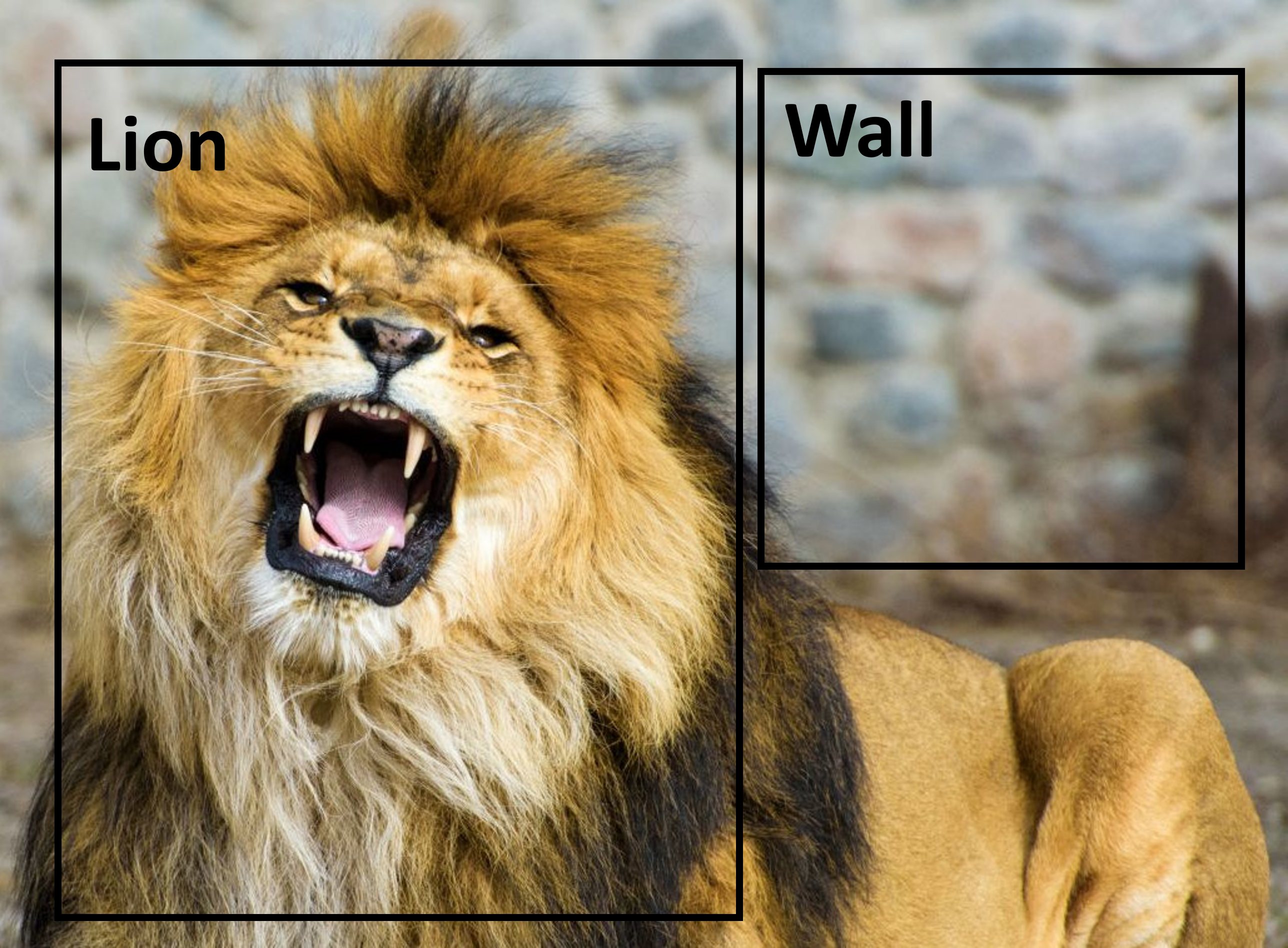}
    \ar{\caption{Image labeled with the presence of a lion and a wall. Classification tasks performed by sub-symbolic models can benefit from logic inference on additional symbolic knowledge.}\label{fig:lion}}
\end{figure}

\ar{In the following, we introduce more formally the framework where our work takes place.} Let us consider a multi-task learning problem where $\B P=(P_1,\ldots,P_J)$ denotes the vector of real-valued functions (task functions) to be determined.
% with $J>0$ number of tasks.
Given the set $\mathcal{X}\subseteq\mathbb{R}^n$ of available data, a supervised learning problem can be generally formulated as
$\min_{\B P}\mathcal{L}(\mathcal{X},\B P)$
where $\mathcal{L}$ is a positive-valued functional denoting a certain loss.
In our framework, we assume that the task functions are FOL predicates and all the available knowledge about these predicates, including supervisions, is collected into a knowledge base $KB=\{\psi_1,\ldots,\psi_H\}$ of FOL formulas.
The learning task is then expressed as:
\[
\min_{\B P}\mathcal{L}(\mathcal{X}, KB, \B P) 
\]
The link between FOL knowledge and learning was also presented e.g. in \cite{marra2019lyrics} and it can be summarized as follows.
\begin{itemize}
	\item Each \emph{Individual} is an element of a specific domain, which can be used to ground the predicates defined on such a domain. Any replacement of variables with individuals for a certain predicate is called \emph{grounding}.
	\item 
	\emph{Predicates} express the truth degree of some property for an individual (unary predicate) or group of individuals (n-ary predicate). In particular, this paper will focus on learnable predicate functions implemented by (deep) neural networks, \ar{but other models can also be used}. FOL \textit{functions} \ar{can be included and learned in a similar fashion \cite{marra2019constraint}}. \ar{However, in this presentation,} function-free FOL is used to keep the notation simpler.
	\item \emph{Knowledge Base} (KB) is a collection of FOL formulas expressing the learning task. The integration of learning and logical reasoning is achieved by compiling the logical rules into continuous real-valued constraints  correlating all the defined elements and enforcing some \ar{expected behavior on them}.
\end{itemize}
Given any rule in KB, individuals, predicates, logical connectives and quantifiers can all be seen as nodes of an \textit{expression tree}~\cite{diligenti2018delbp}. \ar{Then, the translation into a functional} constraint corresponds to a post-fix visit of the expression tree, consisting of the following steps:
\begin{itemize}
	\item visiting a \textit{variable} substitutes the variable with the corresponding feature representation of the individual to which the variable is currently assigned;
	\item visiting a \textit{predicate} computes the output of the predicate with the current input groundings;
	\item visiting a \textit{connective} combines the grounded predicate values by means of the real-valued operation associated to the connective;
	\item visiting a \textit{quantifier} aggregates the outputs of the expressions obtained for the single individuals (variable groundings).
\end{itemize}
\begin{figure}[t]
    \centering
    \includegraphics[width=10cm]{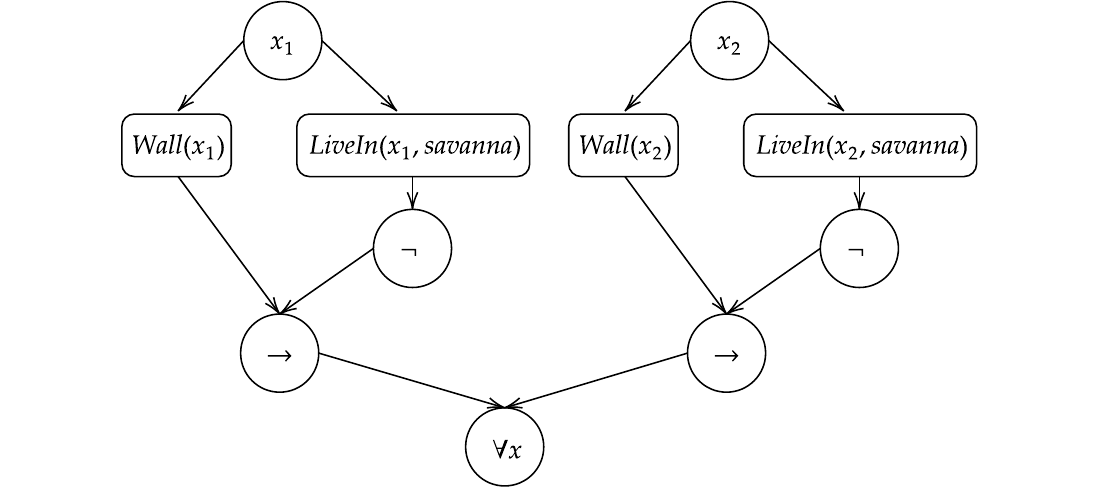}
    \ar{\caption{The expression tree corresponding to $\forall x\ Wall(x)\rightarrow\neg LiveIn(x,savanna)$ for the domain of constants $\mathcal{X}=\{x_1, x_2\}$.}
    \label{fig:extree}}
\end{figure}
Thus, the compilation of the expression tree allows us to convert \ar{a} formula into \ar{a} real-valued function, represented by a computational graph. \ar{The different  functions corresponding to predicates} are \ar{composed (i.e. aggregated)} by means of the truth-functions corresponding to connectives and quantifiers. Given a formula $\varphi$, \ar{we denote by} $f_\varphi$ its corresponding real-valued functional representation. $f_\varphi$ tightly depends on the chosen t-norm \ar{driving the fuzzy relaxation}.
\ar{The expression tree corresponding to the FOL formula $\forall x\ Wall(x)\rightarrow\neg LiveIn(x,savanna)$ is reported in Figure \ref{fig:extree} as an example}.

\begin{example}
    Given two predicates $P_1,P_2$ and the formula $\varphi(x)=P_1(x)\Rightarrow P_2(x)$, the functional representation of $\varphi$ is given by $f_{\varphi}(x, \B P) =\min\{1,1-P_1(x)+P_2(x)\}$ \ar{and $f_{\varphi}(x, \B P) =\min\{1,P_2(x)/P_1(x)\}$ in the \L ukasiewicz and Product logics, respectively.}
\end{example}

A special note concerns quantifiers. \ar{They aggregate the truth-values of predicates over their corresponding domains}. For instance, according to  \cite{novak2012mathematical}, that first proposed a fuzzy generalization of FOL, the universal and existential quantifiers may be converted as the infimum and supremum over a domain variable (\ar{coinciding with} minimum and maximum when dealing with finite domains). In particular, given a formula $\varphi(x)$ depending on a certain variable $x\in\mathcal{X}$, where $\mathcal{X}$ denotes the \ar{finite set of} available samples for one of the involved predicates in $\varphi$, the fuzzy semantics of the quantifiers is given by:
\[
\begin{array}{ccc}
\psi = \forall x\, \varphi(x) &\longrightarrow& f_\psi(\mathcal{X}, \B P) = \displaystyle\min_{x\in\mathcal{X}} f_{\varphi}(x,\B P)\\
\psi = \exists x\, \varphi(x) &\longrightarrow& \displaystyle f_\psi(\mathcal{X}, \B P) = \max_{x\in\mathcal{X}} f_{\varphi}(x,\B P)
\end{array}
\]
As shown in the next section, this quantifier \ar{relaxation} is not \ar{convenient} for all the t-norms and we propose a more principled approach for the translation.

Once all the formulas in $KB$ are converted into real-valued functions, their distance from satisfaction (i.e. distance from 1-evaluation) can be computed according to a certain decreasing mapping $L$ expressing the penalty for the violation of any constraint.
In order to satisfy all the constraints, 
% Assuming %FL rule independence, 
the learning problem can be formulated as the joint minimization over the single rules using the following loss function factorization:
\begin{equation}\label{eq:loss}
\mathcal{L}(\mathcal{X},KB,\B P) = \sum_{\psi \in KB} \beta_\psi L\big(f_\psi(\mathcal{X},\B P)\big)
\end{equation}
Here any $\beta_\psi$ denotes the weight for the logical constraint $\psi$ in the $KB$, which can be selected via cross-validation or jointly learned \cite{kolb2018learning,marra2019integrating},
%$f_h \equiv
$f_{\psi}$ is the functional representation of the formula $\psi$ according to a certain t-norm fuzzy logic and $L$ is a decreasing function denoting the penalty associated to the distance from satisfaction of formulas, so that $L(1)=0$. 

\ar{As described in Section \ref{sec:related_works}, in this neural-symbolic scenario all the steps involved in the translation of FOL formulas into a loss function are treated separately, involving very heterogeneous choices. In the next section, we show instead that these steps are intrinsically connected and they can be uniformly derived from a unique global choice: the selection of a t-norm generator. }
% This paper shows that the selected semantics of the $f_\psi$ converting a formula $\psi$ and the choice of the penalty function $L$ are intrinsically connected, and they can be both derived by the selection of a t-norm generator.
% %%%%%%%%%%%%%%%%%%%%%%%%%%%%%%%%%%%%

\section{Loss Functions by T-Norms Generators}
\label{sec:Lgen}
% This section presents an extension of the approach
This section presents a generalization of the approach \ar{introduced} in \cite{giannini2019relation}, which was limited to supervised learning. \ar{In this paper, we present a unified principle to translate the fuzzy relaxation of FOL formulas into the loss function of general machine learning tasks. In particular, we study} the mapping of \ar{FOL} formulas into \ar{functional} constraints by means of \ar{continuous Archimedean} t-norm fuzzy logics. We adopt the t-norm generator \ar{to penalize the violation of the constraints}, i.e. we take $L=g$. Moreover, since the quantifiers can be seen as generalized AND and OR over the grounded expressions (see Remark \ref{re:quant}), we show that \ar{by adopting} the same fuzzy conversion \ar{for connectives and quantifiers}, the overall loss function expressed in Equation \ref{eq:loss} only depends on the chosen t-norm generator \ar{$g$}. 
\begin{remark}\label{re:quant}
	Given a formula $\varphi(x)$ defined on the available set of samples $\mathcal{X}=\{x_1,\ldots,x_N\}$, the roles of the quantifiers have to be interpreted as follows:
	\[
	\begin{array}{c}
\forall x\, \varphi(x)\;\simeq\; \varphi(x_1)\mbox{ AND }\ldots\mbox{ AND }\varphi(x_N)\\
	\exists x\, \varphi(x)\;\simeq\; \varphi(x_1)\mbox{ OR }\ldots\mbox{ OR }\varphi(x_N)
	\end{array}
	\]
\end{remark}

\subsection{General Formulas}
\label{sec:general_formulas}

Given a certain formula $\varphi(x)$ depending on a variable $x$ that ranges in the set $\mathcal{X}$ and its corresponding functional representation $f_{\varphi}(x,\B P)$, the conversion of any universal quantifier may be carried out by means of an Archimedean t-norm $T$, while the existential quantifier by a t-conorm. For instance, given
the formula $\psi=\forall x\,\varphi(x)$, we have: 
\begin{equation}\label{eq:constbygen}
f_{\psi}(\mathcal{X},\B P) = g^{-1}\left(\min\left\{g(0^+),\sum_{x\in\mathcal{X}}g\big(f_{\varphi}(x,\B P)\big)\right\}\right)
\end{equation}
where $g$ is a generator of the t-norm $T$.

Since any generator function $g$ is decreasing and $g(1)=0$, a generator is a \ar{suitable} choice to map the fuzzy conversion of a formula into a constraint loss to be minimized. By exploiting the same generator of $T$ as loss function \ar{(i.e. taking $L=g$)} for $\psi=\forall x\,\varphi(x)$ expressed by Equation \ref{eq:constbygen}, we get the following term $L\big(f_\psi(\mathcal{X},\B P)\big)$ to be minimized:
\begin{equation}
L\big(f_\psi(\mathcal{X},\B P)\big) = 
\begin{cases}
\min\left\{g(0^+),\displaystyle\sum_{x\in\mathcal{X}}g(f_{\varphi}(x,\B P))\right\}  & \!\!\mbox{if $T$ is nilpotent}\\
&\\
\displaystyle\sum_{x\in\mathcal{X}}g(f_{\varphi}(x,\B P)) & \mbox{if $T$ is strict}
\end{cases}
\label{eq:loss_form}
\end{equation}
As a consequence, the following result can be provided with respect to the convexity of the loss $L\big(f_\psi(\mathcal{X},\B P)\big)$.
\begin{proposition}\label{prop:genconv}
	If $g$ is a linear function and $f_\psi$ is concave \ar{then} $L\big(f_\psi(\mathcal{X},\B P)\big)$ is convex. 
	If $g$ is a convex function and $f_\psi$ is linear \ar{then} $L\big(f_\psi(\mathcal{X},\B P)\big)$ is convex. 
\end{proposition}
\begin{proof}
	Both the arguments follow since, if $f_{\psi}$ is concave (we recall that a linear function is both concave and convex) and $g$ is a convex non-increasing function defined over a univariate domain, then $g\circ f_{\psi}$ is convex.
\end{proof}
\noindent Proposition \ref{prop:genconv} establishes a general criterion to define convex constraints according to a certain generator depending on the fuzzy conversion $f_{\psi}$ and, in turn, by the logical expression $\psi$. In the following of this section, we show some application cases of this proposition.

So far, we did not make any hypothesis on the formula $\varphi$. In the following, different cases of interest for the main connective of $\varphi$ are reported. Given an additive generator $g$ for a t-norm $T$, additional connectives may be expressed with respect to $g$, as reported by Equation \ref{eq:genres}.
If $P_1,P_2$ are two unary predicate functions sharing the same input domain $\mathcal{X}$, the following formulas yield the following penalty terms,
where we supposed $T$ strict for simplicity:
\[
\begin{array}{rcl}\label{eq:convFOLgen}
\forall x\,P_1(x) &\longrightarrow & \displaystyle\sum_{x \in \mathcal{X}} g(P_1(x))\\
\forall x\,P_1(x)\Rightarrow P_2(x) & \longrightarrow& \displaystyle \sum_{x \in \mathcal{X}} \max\{0, g(P_2(x))-g(P_1(x))\}
\\
\forall x\,P_1(x)\Leftrightarrow P_2(x) & \longrightarrow & \displaystyle \sum_{x \in \mathcal{X}} \mid g(P_1(x))-g(P_2(x)) \mid
\end{array}
\]

% \begin{table*}
% 	$
% 	\tiny
% 	\begin{array}{l}
% 	 \overbrace{g^{-1}\!\left(\!\sum_x g\!\left(\overbrace{g^{-1}\!\left(\!\max\!\left\{0,g(P_3(x))\!-\!g\left(\overbrace{g^{-1}(g(P_1(x))\!+\!g(P_2(x)))}^{conjunction}\right)\!\right\}\!\right)}^{implication}\!\right)\!\right)}^{quantifier} =\displaystyle g^{-1}\left(\sum_x \max\left\{0,g(P_3(x))-g(P_1(x))-g(P_2(x))\right\}\right)\\
% 	\\
%     \end{array}
%     $
%     \caption{Conversion of the formula $\forall x\,P_1(x)\otimes P_2(x)\Rightarrow P_3(x)$ with respect to the selection of the strict t-norm generator $g$.}
%     \label{tab:example_generator}
% \end{table*}

\paragraph{\ar{Examples of Derived Losses}}
\ar{According to the selection of the generator, the same FOL formula can be mapped to different loss functions.
This enables us to design customized losses that are more suitable for a specific learning problem, or to provide a theoretical justification to the losses that are already commonly utilized by the machine learning community. 
Examples \ref{ex:conv}-\ref{ex:implication2} show some application cases. In particular, also the cross-entropy loss (see Example \ref{ex:conv2}) can be justified under the same logical perspective.}

\begin{example}
    \label{ex:conv}
	If $g(x)=1-x$ we get the \L ukasiewicz t-norm, that is nilpotent. Hence, from Equation \ref{eq:loss_form} we get:
	\[
	L\big(f_\psi(\mathcal{X},\B P)\big) = \min\{1,\sum_{x\in\mathcal{X}}(1-(f_\varphi(x,\B P)))\} \ .
	\]
	In case $f_\psi$ is concave (e.g. if $\psi$ belongs to the concave fragment of \L ukasiewicz logic \cite{giannini2018convex}), this function is convex.
	\end{example}

\begin{example}\label{ex:conv2}
	If $g(x)=-\log(x)$ \ar{we get the Product t-norm, that is strict.} \ar{From Equation \ref{eq:loss_form} we get a generalization of the cross-entropy loss:}
	\[
	L\big(f_\psi(\mathcal{X},\B P)\big) = -\sum_{x\in\mathcal{X}}\log(f_\varphi(x)) \ .
	\]
	In case $f_\psi(x)$ is linear (e.g. a literal), this function is convex.
\end{example}

\begin{example}
\label{ex:implication1}
	If $g(x)=\frac{1}{x}-1$, with corresponding strict t-norm $T(x,y)=\frac{xy}{x+y-xy}$, the penalty term that is obtained applying $g$ to the formula \ar{$\psi=$}$\forall x\,P_1(x)\Rightarrow P_2(x)$ is given by
	\[
	L\big(f_\psi(\mathcal{X},\B P)\big) = \sum_{x \in \mathcal{X}} \max\left\{0, \frac{1}{P_2(x)}-\frac{1}{P_1(x)}\right\} \ .
	\]
\end{example}
\begin{example}
    \label{ex:implication2}
	    If $g(x)=1-x^2$, with corresponding nilpotent t-norm $T(x,y)=\min\{1,2-x^2-y^2\}$, we get for \ar{$\psi=\forall x\,P_1(x)\Rightarrow P_2(x)$}
	\[
	L\big(f_\psi(\mathcal{X},\B P)\big)\! =\! \min\left\{ \! 1, \!\sum_{x \in \mathcal{X}} \max\left\{0, (P_1(x))^2-(P_2(x))^2\right\}\!\right\} \ .
	\]
\end{example}

\subsection{Simplication Property}
\ar{An interesting property of the presented formulation consists in the fact that}, in case of compound formulas, several occurrences of the generator may be simplified. 
For instance, the conversion \ar{$f_{\psi}(\mathcal{X},\B P)$} of the formula \ar{$\psi=$}$\forall x\,P_1(x)\otimes P_2(x)\Rightarrow P_3(x)$ with respect to the selection of a strict t-norm generator $g$ becomes:
\[
	\begin{array}{l}
	 \overbrace{g^{-1}\!\left(\!\sum_x g\!\left(\overbrace{g^{-1}\!\left(\!\max\!\left\{0,g(P_3(x))\!-\!g\left(\overbrace{g^{-1}(g(P_1(x))\!+\!g(P_2(x)))}^{conjunction}\right)\!\right\}\!\right)}^{implication}\!\right)\!\right)}^{quantifier} =
	 \\
	 =\displaystyle g^{-1}\left(\sum_x \max\left\{0,g(P_3(x))-g(P_1(x))-g(P_2(x))\right\}\right)
    \end{array}
\]
The simplification expressed on the lower side is general and can be applied to a wide range of logical operators, reducing the required number of applications of $g^{-1}$ to just the one in front of the expression. \ar{In these cases, by applying $L=g$, the overall penalty of the formula can be determined by just evaluating $g$ on the predicate functions and without applying $g^{-1}$. Since $g$ and $g^{-1}$ can be in general affected by numerical issues (e.g. $g=-\log$), this property may allow the implementation of more numerically stable loss functions, totally preserving the initial semantics of the formula.}

However, this property does not hold for all the connectives that are definable upon a certain generated t-norm (see Definition \ref{def:conn}). For instance, $\forall x\;P_1(x)\oplus P_2(x)$ becomes:
\[
g^{-1}\Big(\sum_x g(1-g^{-1}(g(1-P_1(x))+g(1-P_2(x))))\Big)
\]
This suggests to identify the connectives that allow, on one hand the simplification of any occurrence of $g^{-1}$ \ar{in $L\big(f_\psi(\mathcal{X},\B P)\big)$}, and on the other hand the evaluation of $g$ only on grounded predicates. For short, in the following we say that the formulas built upon such connectives have the \emph{simplification property}.
\begin{lemma}\label{lemma:sp}
	Any formula $\varphi$ whose connectives are restricted to $\{\wedge,\vee,\otimes,\Rightarrow,\sim,\Leftrightarrow\}$ has the simplification property.
\end{lemma}
\begin{proof}
	The proof is by induction with respect to the number $l\geq0$ of connectives occurring in $\varphi$.
	\begin{itemize}
		\item If $l=0$, i.e. $\varphi=P_j(x_i)$ for a certain $j\leq J$ and $x_i\in\mathcal{X}$, then $g(f_{\varphi})=g(P_j(x_i))$. Hence $\varphi$ has the simplification property.
		\item If $l=k+1$ then $\varphi=(\alpha\circ\beta)$ for $\circ\in\{\wedge,\vee,\otimes,\Rightarrow,\sim,\Leftrightarrow\}$ and we have the following cases.
		\begin{itemize}
			\item If $\varphi=(\alpha\wedge\beta)$ then we get 
			$g(\min\{f_{\alpha},f_{\beta}\})=\max\{g(f_{\alpha}),g(f_{\beta})\}$. The claim follows \ar{by an inductive hypothesis} on $\alpha,\beta$ whose number of involved connectives is less or equal than $k$.
   \\
   The argument still holds replacing $\wedge$ with $\vee$ and $\min$ with $\max$.
			\item If $\varphi=(\alpha\otimes\beta)$ then we get 
			\[
			g(g^{-1}(\min\{g(0^+),g(f_{\alpha})+g(f_{\beta})\})) 
			= \min\{g(0^+),g(f_{\alpha})+g(f_{\beta})\} \ .
			\]
			As in the previous case, the claim follows by inductive hypothesis on $\alpha,\beta$.
			\item The remaining of the cases can be treated in the same way and noting that $\sim\alpha=\alpha\Rightarrow0$.
		\end{itemize}
	\end{itemize}
\end{proof}
The simplification property provides several advantages from an implementation point of view. \ar{First,} it allows the evaluation of the generator function only on grounded predicate expressions and avoids an explicit computation of the pseudo-inverse $g^{-1}$. \ar{Second,} this property provides a general method to implement $n$-ary t-norms, of which universal quantifiers can be seen as a special case since we only deal with finite domains (see Section \ref{sec:discussion}). \ar{Moreover, it is worth to notice that this property does not rely on specific assumptions on the neural models adopted to implement the predicate functions nor on the chosen fuzzy logic exploited for the relaxation. As a result, Lemma \ref{lemma:sp} can be applied in a wide range of cases.}

\ar{Finally,} the simplification property yields an interesting analogy between truth-functions and loss functions. In logic, the truth degree of a formula is obtained by combining the truth degree of its sub-formulas by means of connectives and quantifiers. In the same way, the loss corresponding to a formula that satisfies the simplification property is obtained by combining the losses corresponding to its sub-formulas, while connectives and quantifiers combine losses rather than truth degrees.

\subsection{Manifold Regularization: an example}
\label{sec:example}
Let us consider a simple multi-task classification problem where two objects $A,B$ must be detected in a set of input images $\mathcal{I}$, represented as a set of features.
The learning task consists in determining the predicates $P_A(i)$, $P_B(i)$, which return true if and only if the input image $i$ is predicted to contain the object $A$, $B$, respectively.
The positive supervised examples are provided as two sets (or equivalently their membership functions) {$\ar{\mathcal{P}}_A \subset \mathcal{I}$, $\ar{\mathcal{P}}_B \subset \mathcal{I}$ with the images known to contain the object $A,B$, respectively.
The negative supervised examples for $A,B$ are instead provided as two sets $\ar{\mathcal{N}}_A \subset \mathcal{I}$, $\ar{\mathcal{N}}_B \subset \mathcal{I}$.
Furthermore, the location where the images have been taken is assumed to be known, and a predicate $SameLoc(i_1, i_2)$ is used to express whether two images $i_1,i_2$ have been taken in the same location. \ar{Finally, we assume that two images} taken in the same location are likely to contain the same object.  This knowledge about the environment can be enforced via \emph{Manifold Regularization}, which regularizes the classifier outputs over the manifold built by the image co-location defined via the $SameLoc$ predicate. 

\begin{table}[t]
\small
\caption{Example of a learning task expressed using FOL.}
	\label{tab:dfl_example}
	\centering
	% 	\begin{tabular}{l}
	% 		\hline
	% 		$\forall i_1,i_2: SameLoc(i_1,i_2) \land A(i_1) \Rightarrow A(i_2)$\\
	% 		$\forall i_1,i_2: SameLoc(i_1,i_2) \land B(i_1) \Rightarrow B(i_2)$\\
	% 		$\forall i: (P_A(i) \Rightarrow A(i)) \land (N_A(i) \Rightarrow \lnot A(i))$\\
	% 		$\forall i:(P_B(i) \Rightarrow B(i)) \land (N_B(i) \Rightarrow \lnot B(i))$\\
	% 		$P_A(i10)=1, ~~ P_A(i101)=1, ~~N_B(i11)=1, ~~ P_B(i103)=1$\\
	% 		$SameLoc(i23, i60)=1$\\
	% 		\hline
	% 	\end{tabular}
	\begin{tabular}{l}
		\hline
		$\forall i_1,i_2: SameLoc(i_1,i_2) \Rightarrow (P_A(i_1) \Leftrightarrow P_A(i_2))$\\
		$\forall i_1,i_2: SameLoc(i_1,i_2) \Rightarrow (P_B(i_1) \Leftrightarrow P_B(i_2))$\\
		$\forall i: (\ar{\mathcal{P}}_A(i) \Rightarrow P_A(i)) \land (\ar{\mathcal{N}}_A(i) \Rightarrow \lnot P_A(i))$\\
		$\forall i:(\ar{\mathcal{P}}_B(i) \Rightarrow P_B(i)) \land (\ar{\mathcal{N}}_B(i) \Rightarrow \lnot P_B(i))$\\
		$\ar{\mathcal{P}}_A(i\_10)=1, ~~ \ar{\mathcal{P}}_A(i\_101)=1, ~~\ar{\mathcal{N}}_A(i\_11)=1, ~~ \ar{\mathcal{P}}_B(i\_103)=1$\\
		$SameLoc(i\_23, i\_60)=1$\\
		\hline
	\end{tabular}
\end{table}

The \ar{overall knowledge on} this learning task can be expressed using FOL via the statement declarations shown in
Table~\ref{tab:dfl_example}, where it was assumed that images $i\_23, i\_60$ have been taken in the same location and it holds that $\ar{\mathcal{P}}_A=\{i\_10,i\_101\}, ~\ar{\mathcal{P}}_B=\{i\_103\}, ~ \ar{\mathcal{N}}_A=\{i\_11\}$ and $\ar{\mathcal{N}}_B=\emptyset$.
The statements define the constraints that the learners must respect on all the available samples, expressed as FOL rules. Please note that also the fitting of the supervisions on specific input images are expressed as constraints.

Given the selection of a strict generator $g$ and a set of images $I\ar{\subseteq\mathcal{I}}$, the FOL knowledge in Table \ref{tab:dfl_example} is compiled into the following optimization task:
\[
\begin{array}{ll}
\arg \displaystyle\min_{\B P}
&\beta_1 \displaystyle\sum_{i \in \ar{\mathcal{P}}_A} g(P_A(i)) + \beta_2 \displaystyle\sum_{i \in \ar{\mathcal{N}}_A} g(1-P_A(i)) +\\
&\beta_3 \displaystyle\sum_{i \in \ar{\mathcal{P}}_B} g(P_B(i)) +\beta_4 \displaystyle\sum_{i \in \ar{\mathcal{N}}_B} g(1-P_B(i))+ \\
&\beta_5 \displaystyle\sum_{(i_1,i_2) \in I_{sl}} \mid g(P_A(i_1))-g(P_A(i_2))\mid +\\
&\beta_6 \displaystyle\sum_{(i_1,i_2) \in I_{sl}} \mid g(P_B(i_1))-g(P_B(i_2))\mid 
\end{array}
\]
where ${\B P}=\{P_A,P_B\}$, each $\beta_i$ is a  meta-parameter deciding how strongly the $i$-th contribution should be weighted, $I_{sl}$ is the set of image pairs having the same location $I_{sl} = \{(i_1,i_2) : SameLoc(i_1,i_2)\}$. The first four elements of the cost function express the fitting of the supervised data, while the latter two express manifold regularization over co-located images.

% %%%%%%%%%%%%%%%%%%%%%%%%%%%%%%%%%%%%%%%%
\section{Experimental Results}
\label{sec:exp}
The experimental results have been carried out using the \emph{Deep Fuzzy Logic} (DFL) software\footnote{\url{ http://sailab.diism.unisi.it/deep-logic-framework/}} \ar{which allows us to inject prior knowledge in form of a set of FOL formulas into a machine learning task. The formulas are compiled into differentiable constraints using the theory of generators as described in the previous sections. The learning task is then cast into an optimization problem like shown in Section~\ref{sec:example} and, finally, optimized using the} TensorFlow (TF) environment\footnote{\url{https://www.tensorflow.org/}}~\cite{abadi2016tensorflow}.
\ar{In the following section}, it is assumed that each FOL constant corresponds to a tensor storing its feature representation.
\ar{Predicates are mapped to generic functions in the TF computational graph. If the function does not contain any learnable parameter in the graph, it is said to be \textit{given}, otherwise the function/predicate is said to be \textit{learnable}, and its} parameters will be optimized to maximize the constraints satisfaction.
Please note that any learner expressed as a TF computational graph can be transparently \ar{incorporated into DFL}.

\subsection{The Learning Task}
The CiteSeer dataset~\cite{fakhraei2015collective} consists of 3312 scientific papers, each one assigned to one of six classes: Agents, AI, DB, IR, ML and HCI. The papers are not independent as they are connected by a citation network with 4732 links.  
This dataset defines a relational learning benchmark, where it is assumed that the representation of an input document is not sufficient for its classification without exploiting the citation network. \ar{The citation network can be used to inject useful information into the learning task, as it is often true that two papers connected by a citation belong to the same category.}
\begin{figure}[t]
	\centering
	\begin{tabular}{cc}
		\includegraphics[width=.47\textwidth]{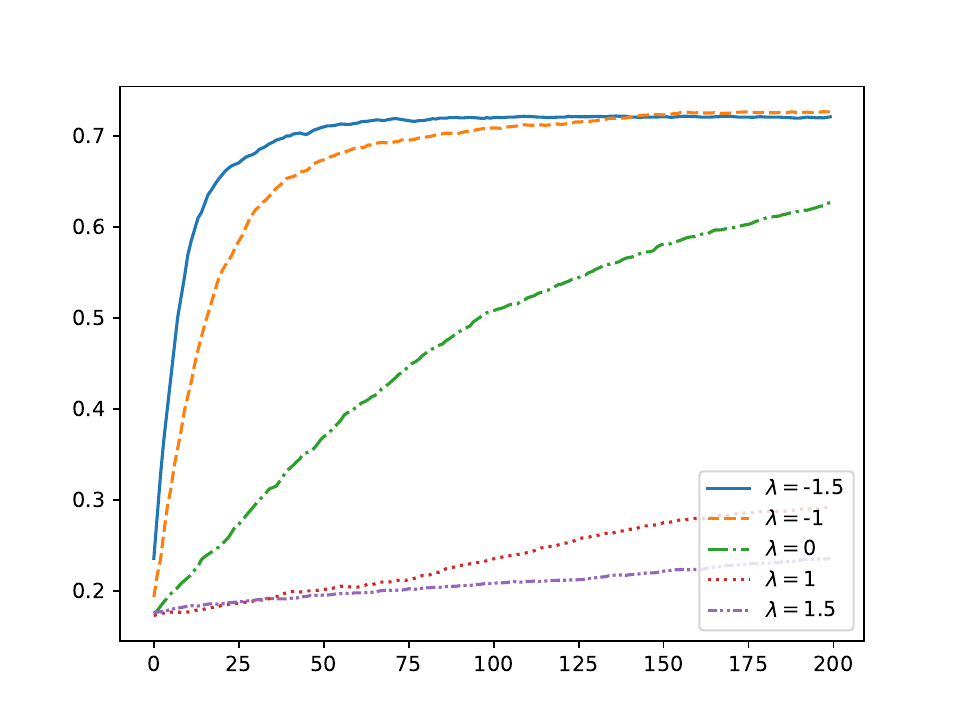}
		&
		\includegraphics[width=.47\textwidth]{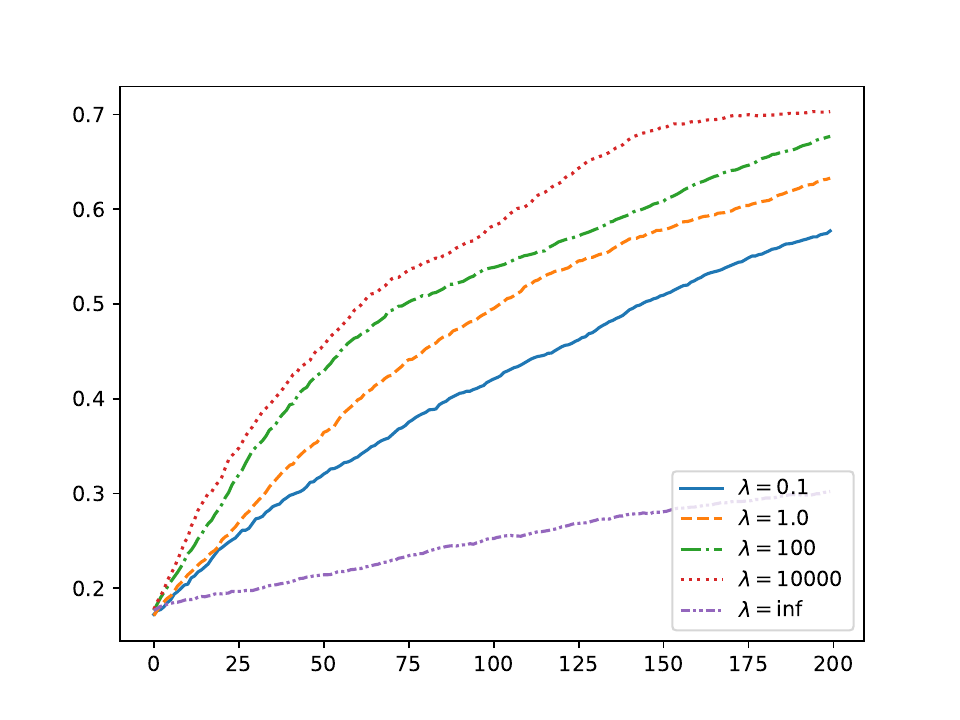}	\vspace{-0.3cm}
		\\
		\small{a. SS - GD} & \small{b. Frank - GD} \\
		\includegraphics[width=.47\textwidth]{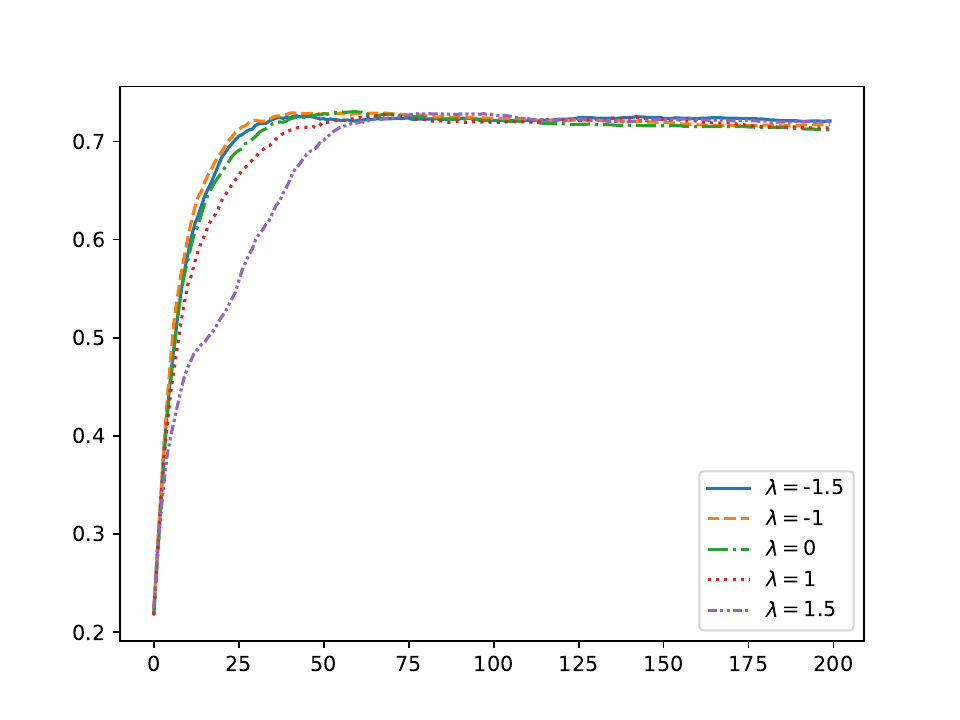}
		&
		\includegraphics[width=.47\textwidth]{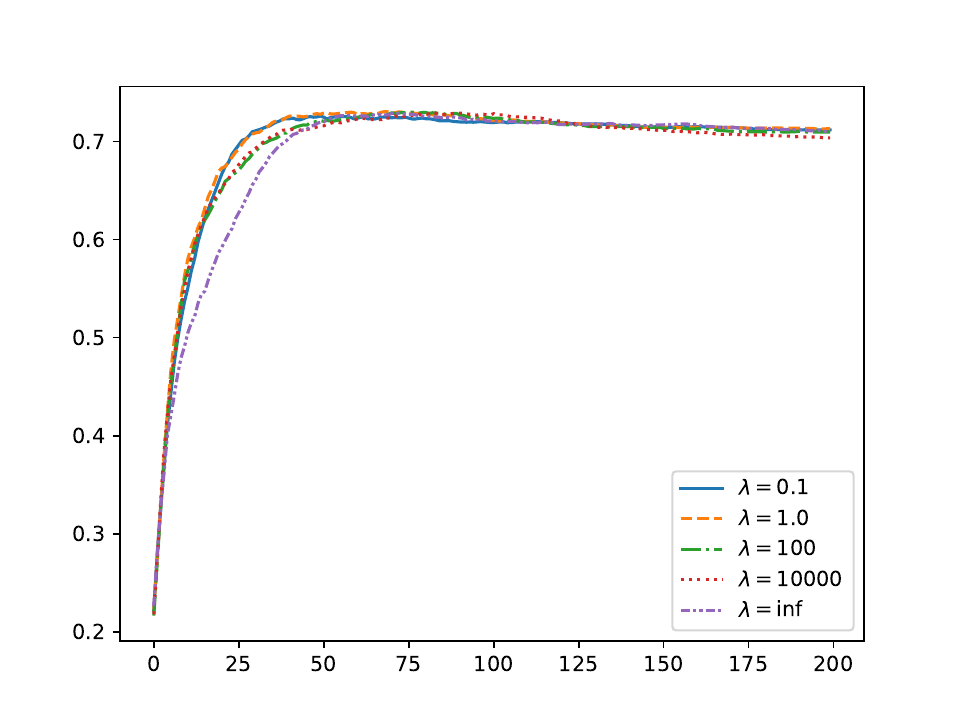}	\vspace{-0.3cm}\\
	    \small{c. SS - Adam} & \small{d. Frank - Adam}
	    \vspace{+0.2cm}
	\end{tabular}
	\caption{Learning Dynamics in terms of test accuracy on a supervised task when choosing different t-norms generated by the parameterized SS and Frank families: \ar{(a.) and (b.) are learning processes optimized with standard gradient descent, while (c.) and (d.) are optimized with Adam~\cite{kingma2014adam}}.}
	\label{fig:learning_dynamics_adam}
\end{figure}
This knowledge can be expressed by providing a general rule of the form:
$\forall x \; \forall y \; Cite(x,y) \Rightarrow \big (P(x) \iff P(y)\big)$, 
where $Cite$ is a binary predicate encoding the fact that $x$ is citing $y$ and $P$ is a task function implementing the membership function of one of the six considered categories. This logical formula expresses \ar{a form of manifold regularization}, which often emerges in relational learning tasks. Indeed, by linking the prediction of two distinct documents, \ar{the behavior} of the underlying task functions is regularized enforcing smooth transition over the manifold induced by the $Cite$ relation.

Each paper is represented via its bag-of-words, which is a vector having the same size of the vocabulary with the i-th element having a value equal to 1 or 0, depending on whether the i-th word in the vocabulary is present or \ar{absent} in the document, respectively. \ar{In particular, the dictionary in this task} consists of 3703 unique words. The set of input document representations is indicated by $X$, which is split into a train and test set $X_{tr}$ and $X_{te}$, respectively. The percentage of documents in the two splits is varied across the different experiments.
The six task functions $P_i$ with $i \in \{Agents, AI, DB, IR, ML , HCI\}$ are bound to the six outputs of a Multi-Layer-Perceptron (MLP) implemented in TF. The neural architecture has 3 hidden layers, with 100 ReLU units each, and softmax activation on the output. Therefore, the task functions share the weights of the hidden layers in such a way that all of them can exploit a common hidden representation.
The $Cite$ predicate is a given
%(fully known a prior)
function, which outputs 1 if the document passed as first argument cites the document passed as second argument, otherwise it outputs 0. Furthermore, \ar{an additional given predicate $\mathcal{P}_i$} is defined for each $P_i$, such that it outputs $1$ \ar{if and only if} $x$ is a positive example for the category $i$ (i.e. it belongs to that category). \ar{$\mathcal{P}_i$ is a supervision predicate, which easily allows us to introduce a supervised signal using FOL ( Section \ref{sec:general_formulas}).}
A manifold regularization learning problem~\cite{belkin2006manifold} can be defined by providing, $\forall i \in \{Agents, AI, DB, IR, ML, HCI\}$, the following two FOL formulas:
\begin{eqnarray}
\forall x \; \forall y \quad & Cite(x,y) \Rightarrow \big (P_{i}(x) \Leftrightarrow P_{i}(y)\big ) \label{eq:manifold} \\
\forall x \quad & \ar{\mathcal{P}_i(x)} \Rightarrow P_{i}(x) \label{eq:supervision}
\end{eqnarray}
where only positive supervisions have been provided because the trained networks for this task employ a softmax activation function on the output layer, which has the effect of imposing mutually exclusivity among the task functions, reinforcing the positive class and discouraging all the others. %While \ar{this behavior} could have been trivially expressed using logic, this network architecture provides a principled baseline to compare against and it was therefore used across all the experiments for this dataset.

\ar{DFL} allows the user to specify the weights of \ar{the} formulas, which are treated as hyperparameters. \ar{Since we use two formulas per predicate}, the weight of the formula expressing the fitting of the supervisions (Equation~\ref{eq:supervision}) is set to a fixed value equal to 1, while the weight of the manifold regularization rule (Equation~\ref{eq:manifold}) is cross-validated from the grid of values $\{0.1, 0.01, 0.006, 0.003, 0.001,0.0001\}$.

\subsection{Results}
The experimental results measure different \ar{aspects of the integration of the prior logic knowledge into a supervised learning task. In particular, different experiments have been designed to track the speed at which the training process converges to the best solution, and how the classification accuracy changes with a variable amount of training data.} % Finally, these measures are compared against purely supervised learning approaches.

\paragraph{Training Convergence Rate}
This experimental setup aims at verifying the relation between the choice of the generator and \ar{the} speed of convergence of the training process. In particular, a simple supervised learning setup is assumed for this experiment, \ar{where the learning task enforces the fitting of the supervised examples as defined by Equation~\ref{eq:supervision}}. The training and test sets are composed of 90\% and $10\%$ of the total number of papers, respectively. Two parameterized families of t-norms have been considered: the SS family (Definition \ref{ex:ss}) and the Frank family (Definition \ref{ex:frank}). Their parameter $\lambda$ was varied to construct classical t-norms for some special values of the parameter but also to evaluate some intermediate ones. In order to keep a clear intuition behind the results, optimization was initially carried out using \ar{simply} a Gradient Descent schema with a fixed learning rate equal to $\eta = 10^{-5}$. Results are shown in Figures (\ref{fig:learning_dynamics_adam}-a) and (\ref{fig:learning_dynamics_adam}-b): it is evident that strict t-norms tend to learn faster than nilpotent ones by penalizing more strongly highly unsatisfied ground formulas. This difference is \ar{significant, although slightly reduced, when leveraging} the state-of-the-art dynamic learning rate optimization algorithm Adam~\cite{kingma2014adam} as shown in Figures~\ref{fig:learning_dynamics_adam}-c and \ref{fig:learning_dynamics_adam}-d. %Finally, it emerges from the experiments that larger values of $\lambda$ tend to be more unstable due to their exponential \ar{behavior} and this can lead to lower performances at convergence.
This finding is consistent with the empirically well known fact that the cross-entropy loss performs well in supervised learning tasks for deep architectures, because it is effective in avoiding gradient vanishing in deep architectures. The cross-entropy loss corresponds to a strict generator with $\lambda=0$ and $\lambda=1$ in the SS and Frank families, respectively. This selection corresponds to a fast and stable converging solution when paired with Adam, while there are faster converging solutions when using a fixed learning rate.

\paragraph{Classification Accuracy}
The effect of the selection of the generator on the classification accuracy is tested on a classification \ar{problem with manifold regularization. This learning task works} in a transductive setting, where all the data is available at training time, even if only the training set supervisions are used during learning. In particular, the data is split into different datasets, where $\{10\%, 25\%, 50\%, 75\%, 90\%\}$ of the available data is \ar{used as a test set}, while the remaining data \ar{is used as training set}. \ar{The fitting of the supervised data defined by Equation \ref{eq:supervision} is enforced for the training data during the learning process, whereas} manifold regularization (Equation \ref{eq:manifold}) can be enforced on all the available data. \ar{The Adam optimizer and the SS family of parametric t-norms have been employed in this experiment}. Table \ref{tab:manifold} shows the average test accuracy and its standard deviation over 10 different samples of the train/test splits. As expected, all generator selections improve the final accuracy over what obtained by pure supervised learning, as manifold regularization brings relevant information to the learner.
\begin{table}[t]
	\small
 \caption{Test accuracy of collective classification in a transductive setting on the Citeseer dataset for different percentages of available training data and different selections of the parameter $\lambda$ of the SS generator family.} %The manifold regularization constraint provides a significant improvement over the purely-supervised baseline. While strict t-norms provide better accuracy when learning with a large number of supervisions, nilpotent t-norms better enforce constraints in tasks where the supervisions are scarce.
	\label{tab:manifold}
	\centering
		\begin{tabular}{cccccc}
			\hline
			\multirow{2}{*}{$\%$ Test} & \multirow{2}{*}{$\lambda$} & \multicolumn{2}{c}{Supervised} & \multicolumn{2}{c}{Manifold} \\
			\cline{3-6} 
			&&Avg Accuracy & Stddev & Avg Accuracy & Stddev \\
			\hline
			\multirow{5}{*}{10\%} & -1.5 & 72.44 & 0.8  & 79.07 & 1.07\\
			& -1.0 & 72.26 & 0.96  & 79.37 & 0.68\\
			& 0.0 & 71.63 & 0.74  & 79.37 & 0.84\\
			& 1.0 & 71.57 & 0.88  & 78.58 & 0.69\\
			& 1.5 & 71.93 & 1.11  & 77.77 & 0.89\\
			\hline
			\multirow{5}{*}{25\%} & -1.5 & 72.22 & 0.46  & 77.17 & 0.70\\
			& -1.0 & 72.02 & 0.52  & 77.51 & 0.72\\
			& 0.0 & 71.35 & 0.56  & 77.39 & 0.50\\
			& 1.0 & 71.22 & 0.47  & 77.36 & 0.64\\
			& 1.5 & 71.51 & 0.77  & 76.41 & 0.57\\
			\hline
			\multirow{5}{*}{50\%} & -1.5 & 70.94 & 0.56  & 75.52 & 0.46\\
			& -1.0 & 70.98 & 0.51  & 76.16 & 0.32\\
			& 0.0 & 70.49 & 0.52  & 75.71 & 0.39\\
			& 1.0 & 70.07 & 1.71  & 76.39 & 0.46\\
			& 1.5 & 70.09 & 0.47  & 75.97 & 0.55\\
			\hline
			\multirow{5}{*}{75\%} & -1.5 & 67.06 & 0.58  & 72.25 & 0.50\\
			& -1.0 & 66.96 & 0.44  & 72.48 & 0.50\\
			& 0.0 & 67.02 & 0.54  & 72.73 & 0.61\\
			& 1.0 & 66.34 & 0.29  & 73.77 & 0.34\\
			& 1.5 & 65.93 & 0.64  & 73.37 & 0.37\\
			\hline
			\multirow{5}{*}{90\%} & -1.5 & 61.09 & 0.78  & 66.02 & 2.51\\
			& -1.0 & 61.59 & 0.44  & 67.24 & 1.72\\
			& 0.0 & 61.52 & 0.33  & 68.60 & 0.75\\
			& 1.0 & 61.31 & 0.52  & 70.69 & 0.52\\
			& 1.5 & 61.17 & 0.84  & 70.32 & 0.89\\
			\hline
		\end{tabular}
\end{table}

Table \ref{tab:manifold} also shows the test accuracy when the parameter $\lambda$ of the SS parametric family is selected from the grid $\{-1.5, -1, 0, 1, 1.5\}$, where values of $\lambda \le 0$ move across strict t-norms (with $\lambda=0$ being the product t-norm), and values greater than $0$ move across nilpotent t-norms (with $\lambda=1$ being the \L ukasiewicz t-norm). 
Strict t-norms seem to provide slightly better performances than nilpotent ones on supervised \ar{tasks} \ar{for the vast majority of the splits}. However, this does not hold \ar{in manifold regularization learning tasks} and a limited number of supervisions, where nilpotent t-norms perform better. An explanation of this behavior can be found in the different nature of the two constraints. %, i.e. the supervision constraint of Equation \ref{eq:supervision} and the manifold regularization constraint of Equation \ref{eq:manifold}.
Indeed, while supervisions provide hard \ar{constraints} that need to be strongly satisfied, manifold regularization is a general soft rule, which should allow exceptions. When the number of supervision is small and manifold regularization drives the learning process, the milder behavior of nilpotent t-norms \ar{performs} better, as it more closely models the semantics of the prior knowledge.
Finally, it is worth noticing that very strict t-norms (e.g. $\lambda = -1.5$ in the \ar{considered} experiment) provide \ar{higher} standard deviations compared to other t-norms, especially in the manifold regularization setup. \ar{This provides some evidence} of a trade-off between the improved learning speed provided by strict t-norms and \ar{the introduced} training instability due to their extremely non-linear behavior.

\paragraph{Competitive Evaluation}
Table~\ref{tab:competitors} compares the accuracy of the selected neural model (NN) trained only with the supervised constraint against other two content-based classifiers, namely logistic regression (LR) and Naive Bayes (NB). These baseline classifiers have been compared against collective classification approaches using the citation network data: Iterative Classification Algorithm (ICA)~\cite{neville2000iterative} and Gibbs Sampling (GS)~\cite{lu2003link} applied on top of the output of the LR and NB content-based classifiers.
\begin{table}[t]
    \small
    \caption{Comparison of the test accuracy on the Citeseer dataset obtained by content based and relational classifiers against supervised and relational learning expressed using DFL. All reported results are computed as average over $10$ random splits of the train and test data. The bold number indicates the best performer and a statistically significant improvement over the competitors.}
	\label{tab:competitors}
	\centering
	\begin{tabular}{lc}
			\hline
		& Classification\\
		Method & Accuracy\\
			\hline
		Naive Bayes & 74.87 \\
		ICA Naive Bayes &76.83\\
		GS Naive Bayes &76.80\\
		Logistic Regression &73.21\\
		ICA Logistic Regression &77.32\\
		GS Logistic Regression &76.99\\
		Loopy Belief Propagation & 77.59\\
		Mean Field &77.32 \\
		NN & 72.26 \\
		DFL & \textbf{79.37}\\
			\hline
	\end{tabular}
\end{table}
Furthermore, the results are compared against the two top performers on this task: Loopy Belief Propagation (LBP)~\cite{sen2008collective} and Relaxation Labeling through Mean-Field Approach (MF)~\cite{sen2008collective}. Finally, the results of DFL were built by training the same neural network with both supervision and manifold regularization constraints, \ar{for which it was used} a generator from the SS family with $\lambda=-1$. The accuracy values are obtained as an average over 10-folds created by random splits of $90$\% and $10$\% of the data for the train and test sets, respectively. Unlike the other relational approaches that can only be executed at inference time (collective classification), DFL can distill the knowledge in the weights of the neural network. The accuracy results are the highest among all the tested methodologies, in spite of the fact that the neural network trained only on the supervisions performs slightly worse than the other content-based competitors.

\ar{\section{Discussion and Practical Implications}}
\label{sec:discussion}
The presented framework can be contextualized among a new class of learning frameworks, which \ar{exploits the continuous relaxation of FOL to integrate} logic knowledge in the learning process~ \cite{diligenti2017semantic,badreddine2022logic,marra2019lyrics,van2022analyzing}. 

\ar{\paragraph{Ease of design and numerical stability} Previous frameworks in this class require an a-priori definition} of the operators of a given t-norm fuzzy logic. On the other hand, the presented framework requires only the generator to be defined. \ar{This provides two main advantages: a minimum \textit{design effort} and an improved \textit{numerical stability}}. Indeed, it is possible to apply the generator only on grounded atoms by exploiting the simplification property to apply the \ar{penalty function} (generator) to the atoms, whereas all compositions are performed via stable operators (e.g. min,max,sum). On the contrary, the previous FOL relaxations correspond to an arbitrary mix of non-linear operators, which can potentially lead to numerically unstable implementations.

\ar{\paragraph{Tensor-based integration} The presented} framework provides a fundamental advantage in the integration with tensor-based machine learning frameworks like TensorFlow~\cite{abadi2016tensorflow} or PyTorch~\cite{ketkar2017introduction}.
Modern deep learning architectures can be effectively trained by leveraging tensor operations performed via Graphics Processing Units (GPU). However, this ability is conditioned on the possibility of \ar{concisely express the operators in terms of parallelizable operations like sums or products over $n$ arguments, which are often implemented as atomic operations in GPU computing frameworks, without requiring to resort to slow iterative procedures.} Fuzzy logic operators can not be easily generalized to their $n$-ary form. For example, the \L ukasiewicz conjunction $T_L(x,y) = \max\{0,x+y-1\}$ can be generalized to $n$-ary form as $T_L(x_1,x_2, \dots, x_n) = \max\{0,\sum_{i=1}^n (x_i) - n + 1\}$. On the other hand, the general SS t-norm $T^{SS}_{\lambda}(x,y) = (x^\lambda + y^\lambda -1)^{\frac{1}{\lambda}}$, \ar{with} $-\infty < \lambda < 0$, does not have any (similarly simple) generalization and the implementation of the $n$-ary form must resort to an iterative application of the binary form, which is very inefficient in tensor-based computations. Previous frameworks like LTN and SBR had to limit the form of the formulas that can be expressed, or carefully select the t-norms in order to provide efficient $n$-ary implementations. However, the presented framework can express operators in $n$-ary form in terms of the generators. Thanks to the simplification property, $n$-ary operators for any \ar{continuous} Archimedean t-norm can always be expressed as $T(x_1, x_2, \dots, x_n) = g^{-1}(\min\{g(0^+),\sum_{i=1}^n g(x_i)\})$ in general, and $T(x_1, x_2, \dots, x_n) = g^{-1}(\sum_{i=1}^n g(x_i))$ if $T$ is strict.

\paragraph{\ar{Limitations}} %Apart from the benefits described above, it is fair to mention that allowing a user to select a customized loss function, as well as to \ar{leverage} different fuzzy semantics to convert connectives and quantifiers could produce better results in terms of pure performances for specific tasks. 
\ar{Linking the loss function to the desired fuzzy semantics via the single choice of the t-norm generator guarantees logic coherence and simplification properties, but does not guarantee to achieve the highest accuracy for a given task.
Another limitation of this approach is that it may not be directly applicable to neural-symbolic models not relaxing the Boolean formulas using t-norm fuzzy logic operators.}

% %%%%%%%%%%%%%%%%%%%%%%%%%%%%%%%%%%%%%%%%%%%%%%
\section{Conclusions}
\label{sec:conc}
This paper presents a framework to embed prior knowledge expressed as logic statements into a learning task \ar{yielding several important contributions. 
First, we showed how human knowledge in the form of logical rules can be translated into differentiable loss functions used during learning. A \ar{critical} aspect of our approach is that the translation from logic formulas to loss functions is uniquely defined by the choice of a unique operator, i.e. the generator of the corresponding t-norm. This feature clearly distinguishes our approach from the majority of related methods, which are often based on multiple specific choices for each of the fuzzy operators.}
%In particular, it was shown how the choice of the t-norm generator used to convert the logic into a differentiable form defines the resulting loss function used during learning.
\ar{Second, we have shown that the classical loss functions for supervised learning are naturally recovered within the theory, and that the use of parametric t-norm generators allows the definition of entire classes of loss functions with different convergence properties. The choice of the parameter can therefore be guided by the requirements of the specific applications.}
%When restricting the attention to supervised learning, the framework recovers popular loss functions like the cross-entropy loss, and allows to define new loss functions corresponding to the choice of the parameters of t-norm parametric forms. 
\ar{Third, the presented theory has driven to the implementation of a general software simulator, called Deep  Fuzzy  Logic (DFL), which bridges logic reasoning and deep learning using the unifying concept of t-norm generator, as general abstraction to translate any FOL declarative knowledge into an optimization problem solved in TensorFlow.}
\ar{Finally, we designed and implemented multiple experiments in DFL which show how the proposed method allows the definition of new loss functions with better performances both in terms of accuracy and training efficiency. Furthermore, by being able to incorporate logical knowledge seamlessly, our method outperforms several related works on the task of document classification in citation networks.}

%The presented theory has driven to the implementation of a general software simulator, called Deep  Fuzzy  Logic, which bridges logic reasoning and deep learning using the unifying concept of t-norm generator, as general abstraction to translate any FOL declarative knowledge into an optimization problem solved in TensorFlow.
\ar{As future work, we plan to extend the method by allowing the learning of the parameters of the t-norm generator from data. In this regard, casting what presented in this paper within a Bayesian framework \cite{marra2019integrating} is likely a promising direction. Furthermore, we plan to expand the range of applications of DFL to domains like visual question answering~\cite{nsvqa} and structure learning~\cite{de2021statistical}}.
%%%%%%%%%%%%%%%%%%%%%%%%%%%%%%%%%%%%%%%%

%%%%%%%%%%%%%%%%%%%%%%%%%%%%%%%%%%%TEMPLATE STUFF

\section*{Statements and Declarations}

\bmhead{Funding}
 This project has received funding from the European Union's Horizon 2020 research and innovation program under grant agreement No 825619. This work was also supported by TAILOR, a project funded by EU Horizon 2020 research and innovation programme under GA No 952215. Giuseppe Marra is funded by Research Foundation-Flanders (FWO-Vlaanderen, 1239422N).

\bmhead{Competing Interests}
The authors declare that they have no competing interest.

\bibliography{references}

\end{document}